\documentclass[11pt]{article}
\pdfoutput=1
\usepackage{enumerate}
\usepackage{pdfsync}
\usepackage[OT1]{fontenc}

\usepackage[usenames]{color}
\usepackage{smile}
\usepackage[colorlinks,
            linkcolor=red,
            anchorcolor=blue,
            citecolor=blue
            ]{hyperref}
\usepackage{fullpage}
\usepackage{hyperref}
\usepackage[protrusion=true,expansion=true]{microtype}
\usepackage{pbox}
\usepackage{setspace}
\usepackage{tabularx}
\usepackage{float}
\usepackage{wrapfig,lipsum}
\usepackage{enumitem}
\usepackage{microtype}
\usepackage{graphicx}
\usepackage{subfigure}
\usepackage{enumerate}
\usepackage{enumitem}
\usepackage{pgfplots}
\usepackage{booktabs} % for professional tables
\usepackage{pifont}
\usepackage{mathtools}

\newcommand{\la}{\langle}
\newcommand{\ra}{\rangle}
\usepackage{enumitem}
\newcommand{\up}{\textnormal{up}}
\newcommand{\loc}{\textnormal{loc}}
\newcommand{\ser}{\textnormal{ser}}
\newcommand{\all}{\textnormal{all}}
\newcommand{\regret}{\mathrm{Regret}}

\makeatletter
\newcommand*{\rom}[1]{\expandafter\@slowromancap\romannumeral #1@}
\makeatother

\newcommand{\alglinelabel}{%
  \addtocounter{ALC@line}{-1}% Reduce line counter by 1
  \refstepcounter{ALC@line}% Increment line counter with reference capability
  \label% Regular \label
}

\ifdefined\final
\usepackage[disable]{todonotes}
\else
\usepackage[textsize=tiny]{todonotes}
\fi
\title{\huge A Simple and Provably Efficient Algorithm for Asynchronous Federated Contextual Linear Bandits}

\author{Jiafan He\thanks{Equal Contribution} \thanks{Department of Computer Science, University of California, Los Angeles, CA 90095, USA; e-mail: {\tt jiafanhe19@ucla.edu}}
\and Tianhao Wang$^*$\thanks{Department of Statistics and Data Science, Yale University, New Haven, CT 06520, USA; e-mail: {\tt tianhao.wang@yale.edu}}
\and Yifei Min$^*$\thanks{Department of Statistics and Data Science, Yale University, New Haven, CT 06520, USA; e-mail: {\tt yifei.min@yale.edu}}
\and Quanquan Gu\thanks{Department of Computer Science, University of California, Los Angeles, CA 90095, USA; e-mail: {\tt qgu@cs.ucla.edu}}
}

\begin{document}
\date{}
\maketitle

\begin{abstract}
We study federated contextual linear bandits, where $M$ agents cooperate with 
each other to solve a global contextual linear bandit problem with the help of a central server. 
We consider the asynchronous setting, where all agents work independently and 
the communication between one agent and the server will not trigger other agents' communication. 
We propose a simple algorithm named \texttt{FedLinUCB} based on the principle of optimism. 
We prove that the regret of \texttt{FedLinUCB} is bounded by $\tilde O(d\sqrt{\sum_{m=1}^M T_m})$ 
and the communication complexity is $\tilde{O}(dM^2)$, where $d$ is the 
dimension of the contextual vector and $T_m$ is the total number of interactions 
with the environment by $m$-th agent. 
To the best of our knowledge, this is the first provably efficient algorithm 
that allows fully asynchronous communication for federated contextual linear 
bandits, while achieving the same regret guarantee as in the single-agent setting. 
\end{abstract}

\section{Introduction}
Contextual linear bandit is a canonical model in sequential decision making with 
partial information feedback that has found vast applications in real-world 
domains such as recommendation systems~\citep{li2010contextual,li2010exploitation,
gentile2014online,li2020federated}, clinical trials~\citep{wang1991sequential,
durand2018contextual} and economics \citep{jagadeesan2021learning,li2022rate}. 
Most existing works on contextual linear bandits focus on either the single-agent 
setting~\citep{auer2002using,abe2003reinforcement,dani2008stochastic,li2010contextual,
rusmevichientong2010linearly,chu2011contextual,abbasi2011improved,agrawal2013thompson} 
or multi-agent settings where communications between agents are instant and 
unrestricted~\citep{cesa2013gang,li2016collaborative,wu2016contextual,li2021unifying}. 
Due to the increasing amount of data being distributed across a large number of 
local agents (e.g., clients, users, edge devices), federated 
learning~\citep{mcmahan2017communication,karimireddy2020mime} has become an 
emerging paradigm for distributed machine learning, where agents can jointly
learn a global model without sharing their own localized data.
This motivates the development of distributed/federated linear 
bandits~\citep{wang2019distributed,huang2021federated, li2021asynchronous}, 
which enables a collection of agents to cooperate with each other to solve a 
global linear bandit problem while enjoying performance guarantees comparable to 
those in the classical single-agent centralized setting.

However, most existing federated linear bandits algorithms are limited to the 
synchronous setting~\citep{wang2019distributed, dubey2020differentially,huang2021federated}, 
where all the agents have to first upload their local data to the server upon 
the request of the server, and the agents will download the latest data from the 
server after all uploads are complete. 
This requires full participation of the agents and global synchronization 
mandated by the server, which is impractical in many real-world application scenarios. 
The only notable exception is~\citet{li2021asynchronous}, where an asynchronous 
federated linear bandit algorithm is proposed. Nevertheless, in their algorithm, 
the upload by one agent may trigger the download from the server to all other agents. 
Therefore, the communications between different agents and the server are not 
totally independent. 
Moreover, their regret guarantee relies on a stringent regularity assumption on 
the contexts, which basically requires the contexts to be stochastic rather than 
adversarial as in standard contextual linear bandits. 
Therefore, how to design a truly asynchronous contextual linear bandit algorithm 
remains an open problem.

In this work, we resolve the above open problem by proposing a simple algorithm 
for asynchronous federated contextual linear bandits over a star-shaped 
communication network. 
Our algorithm is based on the principle of optimism~\citep{abbasi2011improved}. 
The communication protocol of our algorithm enjoys the following features:
(i) Each agent can decide whether or not to participate in each round. 
Full participation is not required, thus it allows temporarily offline agents;
and (ii) the communication between each agent and the server is asynchronous and 
totally independent of other agents. 
There is no need of global synchronization or mandatory coordination by the server.
In particular, the communication between the agent and the server is triggered 
by a matrix determinant-based criterion that can be computed independently by each agent. 
Our algorithm design not only allows the agents to independently operate and 
synchronize with the server, but also ensures low communication complexity 
(i.e., total number of rounds of communication between agents and the server) 
and low switching cost (i.e., total number of local model updates for all 
agents)~\citep{abbasi2011improved}.  

While being simple, our algorithm design introduces a challenge in the regret analysis. 
Since the order of the interaction between the agent and the environment is not 
fixed, standard martingale-based concentration inequality cannot be directly applied. 
Specifically, this challenge arises due to the mismatch between the partial data 
information collected by the central server and the true order of the data 
generated from the interaction with the environment, as is explained in detail 
in Section~\ref{sec: theory} and illustrated by Figure~\ref{fig:filtration}.
We address this challenge by a novel proof technique, which first establishes 
the local concentration of each agent's data and then relates it to the 
``virtual'' global concentration of all data via the determinant-based criterion.
Based on this proof technique, we are able to obtain tight enough confidence 
bounds that leads to a nearly optimal regret.

\paragraph{Main contributions.} 
Our contributions are highlighted as follows:
\begin{itemize}[leftmargin=*]
\item We devise a simple algorithm named \texttt{FedLinUCB} that achieves 
near-optimal regret, low communication complexity and low switching cost 
simultaneously for asynchronous federated contextual linear bandits.
In detail, we prove that our algorithms achieves a near-optimal $\tilde O(d\sqrt{T})$ 
regret with merely $\tilde O(dM^2)$ total communication complexity and 
$\tilde O(dM^2)$ total switching cost. Here $M$ is the number of agents, $d$ is
the dimension of the context and $T = \sum_{m=1}^M T_m$ is the total number of 
rounds with $T_m$ being the number of rounds that agent $m$ participates in. 
When degenerated to single-agent bandits, the regret of our algorithm matches 
the optimal regret $\tilde O(d\sqrt{T})$~\citep{abbasi2011improved}. 
    
\item We also prove an $\Omega(M/\log (T/M))$ lower bound for the communication 
complexity. 
Together with the $ O(dM^2)$ upper bound of our algorithm, it suggests that 
there is only an $\tilde O(dM)$ gap between the upper and lower bounds of the 
communication complexity. 
    
\item We identify the issue of ill-defined filtration caused by the unfixed 
order of interactions between agents and the environment, which is absent in 
previous synchronous or single-agent settings.
We tackle this unique challenge by connecting the local concentration of each 
local agent and the global concentration of the aggregated data from all agents. 
We believe this proof technique is of independent interest for the analysis of 
other asynchronous bandit problems. 
\end{itemize}

\paragraph{Notation.}  
We use lower case letters to denote scalars, and lower and upper case bold face
letters to denote vectors and matrices respectively. For any positive integer $n$, we denote the set $\{1, 2, \ldots, n\}$ by $[n]$.
We use $\Ib$ to denote the $d\times d$ identity matrix. For two sequences $\{a_n\}$ and $\{b_n\}$, we write $a_n=O(b_n)$ if there exists an absolute constant $C$ such that $a_n\leq Cb_n$. We use $\tilde O$ to hide poly-logarithmic terms.
For any vector $\xb \in \RR^d$ and positive semi-definite 
$\bSigma \in \RR^{d \times d}$, we denote $\|\xb\|_{\bSigma} = \sqrt{\xb^\top \bSigma \xb}$, and by $\det(\bSigma)$ the determinant of $\bSigma$.

\section{Related Work}

We review related work on distributed/federated bandit algorithms based on the 
type of bandits: (1) multi-armed bandits, (2) stochastic linear bandits 
and (3) contextual linear bandits.

\paragraph{Distributed/federated multi-armed bandits.}
There is a vast literature on distributed/federated multi-armed bandits (MABs)
\citep{liu2010distributed,szorenyi2013gossip,landgren2016distributed,
chakraborty2017coordinated,landgren2018social,martinez2019decentralized,
sankararaman2019social,wang2019distributed,wang2020optimal,zhu2021federated}, 
to mention a few. 
However, none of these algorithms can be directly applied to linear bandits, 
needless to say contextual linear bandits with infinite decision sets.

\paragraph{Distributed/federated stochastic linear bandits.}
In distributed/federated stochastic linear bandits, the decision set is fixed 
across all the rounds $t\in[T]$ and all the agents $m\in[M]$. 
\citet{wang2019distributed} proposed the \texttt{DELD} algorithm for distributed 
stochastic linear bandits on both star-shaped network and P2P network. 
\citet{huang2021federated} proposed an arm elimination-based algorithm called 
\texttt{Fed-PE} for federated stochastic linear bandits on the star-shaped network. 
Both algorithms are in the synchronous setting and require full participation of 
the agents upon the server's request.

\paragraph{Distributed/federated contextual linear bandits.}
The contextual linear bandit is more general and challenging than stochastic 
linear bandits, because the decision sets can vary for each $t$ and~$m$.
In this setting, \citet{korda2016distributed} considered a P2P network and 
proposed the \texttt{DCB} algorithm based on the \texttt{OFUL} algorithm in 
\citet{abbasi2011improved}.
\citet{wang2019distributed} considered both star-shaped and P2P communication 
networks and achieved the near-optimal $\tilde O(d\sqrt{T})$ regret in the 
synchronous setting.\footnote{In the original paper of \citet{wang2019distributed}, 
the regret bound is expressed as $\tilde O(d\sqrt{MT})$. 
The $T$ in their paper is equivalent to the $T_m$ in ours, so their $d\sqrt{MT}$ 
should be understood as $d\sqrt{T}$ under our notation.}
\citet{dubey2020differentially} further introduced the differential privacy 
guarantee into the setting of \citet{wang2019distributed}.
\citet{li2022communication} extended distributed contextual linear bandits to 
generalized linear bandits \citep{filippi2010parametric,jun2017scalable} in the 
synchronous setting. 
\citet{li2021asynchronous} proposed the first asynchronous algorithm for federated 
contextual linear bandits with the star-shaped graph and achieve near-optimal 
$\tilde O(d\sqrt{T})$ regret. 
However, their setting is different from ours in two aspects: (1) the upload 
triggered by an agent will lead the server to trigger download possibly for all 
the agents in their setting. 
In contrast, the upload triggered by an agent will only lead to download to the 
same agent in our setting; (2) their regret guarantee relies on a stringent 
regularity assumption on the contexts, which basically requires the contexts to 
be stochastic. 
As a comparison, the contexts in our setting can be even adversarial, which is 
exactly the setting of contextual linear bandits \citep{abbasi2011improved,li2019nearly}.
This difference in the setting makes our algorithm a truly asynchronous contextual 
linear bandit algorithm but also makes our regret analysis more challenging. 

For better comparison, we compare our work with the most related contextual 
linear bandit algorithms in Table~\ref{table: comparison with baseline}. 

\begin{table}[t]
\centering
\begin{tabular}{ccccc}
\toprule
Setting    & Algorithm     & Regret  & Communication & Low-switching  \\
\midrule
\multirow{2}{*}{Single-agent} & \texttt{OFUL}   & \multirow{2}{*}{$d\sqrt{T\log T}$}  
& \multirow{2}{*}{N/A} & \multirow{2}{*}{\ding{51}}\\
& {\small\citep{abbasi2011improved}}& & \\
\hline
Federated & \texttt{DisLinUCB}  & \multirow{2}{*}{$d \sqrt{T}\log^2 T$} 
& \multirow{2}{*}{$d M^{3/2}$} & \multirow{2}{*}{\ding{56}} \\(Sync.) 
& \citep{wang2019distributed} & &\\ 
\hline
Federated & \texttt{Async-LinUCB} 
& \multirow{2}{*}{$d\sqrt{T} \log T$} & \multirow{2}{*}{$d M^2\log T$} 
& \multirow{2}{*}{\ding{56}} \\
(Async.)& \citep{li2021asynchronous}\footnotemark[2] & & \\
\hline
Federated & \texttt{FedLinUCB} & \multirow{2}{*}{ $d\sqrt{T} \log T$}
& \multirow{2}{*}{$d M^2\log T$} & \multirow{2}{*}{\ding{51}}\\
(Async.)& (Our Algorithm~\ref{alg: main}) & & \\
\bottomrule\\
\end{tabular}
\caption{Comparison of our result with baseline approaches for contextual linear bandits. 
Our result achieves near-optimal regret under low communication complexity. 
Here $d$ is the dimension of the context, $M$ is the number of agents, and 
$T = \sum_{m=1}^M T_m$ with each $T_m$ being the number of rounds that agent $m$ 
participates in. 
}
\label{table: comparison with baseline}
\end{table}

\section{Preliminaries}

\footnotetext[2]{The regret guarantee in \citet{li2021asynchronous} relies on a stringent regularity assumption on 
the contexts, which basically requires the contexts to be stochastic rather than 
adversarial as in standard contextual linear bandits. }

\paragraph{Federated contextual linear bandits.} 
We consider the federated contextual linear bandits as follows: 
At each round $t \in [T]$, an arbitrary agent $m_t \in [M]$ is active for participation. 
This agent receives a decision set $D_t \subset \RR^d$, picks an action $\xb_t \in D_t$,
and receives a random reward $r_t$. 
We assume that the reward $r_t$ satisfies $r_t = \langle \xb_t , \btheta^* \rangle + \eta_t$ 
for all $t \in [T]$, where $\eta_t$ is conditionally independent of $\xb_t$ given 
$\xb_{1:t-1}, m_{1:t},r_{1:t-1}$. 
More specifically, we make the following assumption on $\eta_t$, $\btheta^*$ and 
$D_t$, which is a standard assumption in the contextual linear bandit 
literature~\citep{abbasi2011improved,wang2019distributed,dubey2020differentially}.

\begin{assumption}\label{assump: sub gaussian and norm}
The noise $\eta_t$ is $R$-sub-Gaussian conditioning on $\xb_{1:t}$, $m_{1:t}$ 
and $r_{1:t-1}$, i.e.,
\begin{align}
\EE\big[e^{\lambda\eta_t} \mid \xb_{1:t}, m_{1:t},r_{1:t-1}\big] 
\leq \exp(R^2\lambda^2/2), \quad \textnormal{for any}\ \lambda \in \RR. \notag
\end{align}
We also assume that $\|\btheta^*\|_2 \leq S$ and $\|\xb\|_2\leq L$ for all 
action $\xb \in \cD_t$, for all $t \in [T]$. 
\end{assumption}

Notably, we assume $m_t$ can be arbitrary for all $t$, which basically says that 
each agent can decide whether and when to participate or not.\footnote{Without 
loss of generality, we can assume that it cannot happen that more than one 
agent participate at the same time. 
Therefore, there is always a valid order of participation indexed by $t\in[T]$.} 
Our setting is more general than the synchronous 
setting in \citet{wang2019distributed,dubey2020differentially}, which requires a 
round-robin participation of all agents. 

\paragraph{Learning objective.} 
The goal of the agents is to collaboratively minimize the cumulative regret defined as
\begin{align}\label{eq: def of regret}
\regret(T) &\coloneqq \sum_{t=1}^T \Big(\max_{\xb\in D_t}\langle \xb,\btheta^*\rangle 
- \langle \xb_t, \btheta^*\rangle\Big)
=\sum_{t=1}^T \langle \xb_{t}^* - \xb_{t}, \btheta^* \rangle.
\end{align}
To achieve such a goal, we allow the agents to collaborate via communication 
through the central server. 
Below we will explain the details of the communication model.

\paragraph{Communication model.} 
We consider a star-shaped communication network~\citep{wang2019distributed,dubey2020differentially} 
consisting of a central server and $M$ agents, where each agent can communicate 
with the server by uploading and downloading data. 
However, any pair of agents cannot communicate with each other directly.  
We define the communication complexity as the total number of communication 
rounds between agents and the server (counting both the uploads and 
downloads)~\citep{wang2019distributed,dubey2020differentially,li2021asynchronous}. 
For simplicity, we assume that there is no latency in the communication channel. 

We consider the asynchronous setting, where the communication protocol satisfies: 
(1) each agent can decide whether or not to participate in each round. 
Full participation is not required, which allows temporarily offline agents; 
and (2) the communication between each agent and the server is asynchronous and 
independent of other agents without mandatory download by the server. 

\paragraph{Switching cost.}
The notion of switching cost in online learning and bandits refers to the number 
of times the agent switches its policy (i.e., decision rule) 
\citep{kalai2005efficient,abbasi2011improved, dekel2014bandits, ruan2021linear}.
In the context of linear bandits, it corresponds to the number of times the agent 
updates its policy of selecting an action from the decision set \citep{abbasi2011improved}.
Algorithms with low switching cost are preferred in practice since each policy 
switching might cause additional computational overhead.

\section{The Proposed Algorithm}\label{sec:alg}
We propose a simple algorithm based on the principle of optimism that enables 
collaboration among agents through asynchronous communications with the central server.
The main algorithm is displayed in Algorithm~\ref{alg: main}.
For clarity, we first summarize the related notations in Table~\ref{tab:notation}.
\begin{wrapfigure}{L}{0.5\textwidth}
    \vskip -0.1in
\begin{minipage}[t]{0.5\textwidth}
\begin{table}[H]
    \centering
    \begin{tabular}{c|c} 
    \hline\hline
    Notation & Meaning \\ 
    \hline
    $\hat\btheta_{m,t}$ & estimate of $\btheta^*$\\
    $\bSigma_{m,t}, \bbb_{m,t}$ & data used to compute $\hat{\btheta}_{m,t}$\\ 
    $\bSigma^\loc_{m,t}, \bbb^\loc_{m,t}$ & local data for agent $m$\\
    $\bSigma^\ser_{m,t}, \bbb^\ser_{m,t}$ & data stored at the server\\
    \hline\hline
    \end{tabular}
    \vspace{0.2cm}
    \caption{Notations used in Algorithm~\ref{alg: main}.}
    \label{tab:notation}
\end{table}
    \end{minipage}
\end{wrapfigure}

Specifically, in each round $t \in [T]$, agent $m_t$ participates and interacts 
with the environment~(Line~\ref{algmain:agent}).
The environment specifies the decision set $D_t$~(Line~\ref{algmain:decision_set}), 
and the agent $m_t$ selects the action based on its current optimisitic estimate 
of the reward~(Line~\ref{algmain:optimistic_decision}).
Here the bonus term $\beta\|\xb\|_{\bSigma_{m_t,t}^{-1}}$ reflects the uncertainty 
of the estimated reward $\la \hat\btheta_{m_t,t},\xb\ra$ and encourages exploration.
After receiving the true reward $r_t$ from the environment, agent $m_t$ then 
updates its local data~(Line~\ref{algmain:local_update}).

The key component of the algorithm is the matrix determinant-based 
criterion~(Line~\ref{algmain:criterion}), which evaluates the information 
accumulated in current local data.
If the criterion is satisfied, it suggests that the local data would help 
significantly reduce the uncertainty of estimating the model $\btheta^*$.
Therefore, agent $m_t$ will share its progress by uploading the local data to 
the server~(Line~\ref{algmain:upload}) so that it can benefit other agents. 
Then the server updates the global data accordingly~(Line~\ref{algmain:update_server}). 
Afterwards, agent $m_t$ downloads the latest global data from the 
server~(Line~\ref{algmain:download}) and updates its local data and 
model~(Line~\ref{algmain:estimate1}-\ref{algmain:estimate2}).
If the criterion in Line~\ref{algmain:criterion} is not met, then the 
communication between the agent and the server will not be triggered, and the 
local data remains unchanged for agent $m_t$~(Line~\ref{algmain:no_update}).
Finally, all the other inactive agents remain unchanged~(Line~\ref{algmain:inactive_agents}).

\begin{algorithm}[t]
	\caption{Federated linear UCB (\texttt{FedlinUCB})}
	\label{alg: main}
	\begin{algorithmic}[1]\label{algorithm:1}
	\STATE Initialize $\bSigma_{m, 1} = \bSigma^\ser_1 = \lambda \Ib$, 
    $\hat\btheta_{m,1}=0$, $\bbb_{m,0}^\loc =0$ and $\bSigma^\loc_{m,0}=0$ for all $m \in [M]$
	\FOR{round $t = 1, 2, \dots, T$}
	\STATE Agent $m_t$ is active \alglinelabel{algmain:agent}
	\STATE Receive $D_t$ from the environment
	\alglinelabel{algmain:decision_set}
	\STATE Select $\xb_t \leftarrow \argmax_{\xb \in D_t} \langle 
    \hat\btheta_{m_t,t},\xb\rangle + \beta\|\xb\|_{\bSigma_{m_t,t}^{-1}}$ 
    \alglinelabel{algmain:optimistic_decision}
    {\color{blue}\hfill \texttt{/* Optimistic decision */}}
	\STATE Receive $r_t$ from environment
	\alglinelabel{algmain:reward}
	\STATE $\bSigma^\loc_{m_t, t}\leftarrow \bSigma^\loc_{m_t, t-1} 
    + \xb_t \xb_t^\top$, \quad $\bbb^\loc_{m_t, t} \leftarrow 
    \bbb^\loc_{m_t, t-1} + r_t \xb_t$
    \alglinelabel{algmain:local_update} \hfill{\color{blue}\texttt{/* Local update */}}
	\IF{$\det(\bSigma_{m_t, t} + \bSigma^\loc_{m_t, t}) > (1 + \alpha)
    \det(\bSigma_{m_t, t})$} \alglinelabel{algmain:criterion} 
    \STATE Agent $m_t$ sends $\bSigma^\loc_{m_t, t}$ and $\bbb^\loc_{m_t, t}$ to 
    server \alglinelabel{algmain:upload} \hfill{\color{blue}\texttt{/* Upload */}}
	\STATE $\bSigma^\ser_t \leftarrow \bSigma^\ser_t + \bSigma^\loc_{m_t, t}$, \quad
    $\bbb^\ser_t \leftarrow \bbb^\ser_t  + \bbb^\loc_{m_t, t}$\alglinelabel{algmain:update_server}
    \hfill{\color{blue}\texttt{/* Global update */}}
    \STATE $\bSigma^\loc_{m_t, t} \leftarrow 0$,\quad $\bbb^\loc_{m_t, t} \leftarrow 
    0$ \alglinelabel{algmain:reset_local}
    \STATE Server sends $\bSigma^\ser_t$ and $\bbb^\ser_t$ back to agent $m_t$ \alglinelabel{algmain:download} \hfill{\color{blue}\texttt{/* Download */}}
    \STATE $\bSigma_{m_t, t+1} \leftarrow \bSigma^\ser_t$,\quad $\bbb_{m_t, t+1} 
    \leftarrow \bbb^\ser_t$ \alglinelabel{algmain:estimate1}
    \STATE $\hat\btheta_{m_t, t+1} \leftarrow \bSigma_{m_t, t+1}^{-1} \bbb_{m_t, t+1}$ \alglinelabel{algmain:estimate2} \hfill{\color{blue}\texttt{/* Compute estimate */}}
    \ELSE
    \STATE $\bSigma_{m_t, t+1} \leftarrow \bSigma_{m, t}$,\quad $\bbb_{m_t, t+1} 
    \leftarrow \bbb_{m, t}$, \quad $\hat\btheta_{m_t, t+1} \leftarrow \hat\btheta_{m_t, t}$\alglinelabel{algmain:no_update}
	\ENDIF
	\FOR{other inactive agent $m\in [M] \setminus \{m_t\}$}
	\STATE $\bSigma_{m, t+1} \leftarrow \bSigma_{m, t}$,\quad $\bbb_{m, t+1} 
    \leftarrow \bbb_{m, t}$, \quad$\hat\btheta_{m, t+1} \leftarrow \hat\btheta_{m,t}$ 
    \alglinelabel{algmain:inactive_agents}
	\ENDFOR
	\ENDFOR
	\end{algorithmic}
\end{algorithm}

Note that in Algorithm~\ref{alg: main}, the communication between the agent and 
the server (Line~\ref{algmain:upload} and ~\ref{algmain:download}) involves only 
the active agent in that round, which is completely independent of other agents.
This is in sharp contrast to existing algorithms.
For example, in the main algorithm in \citet{li2021asynchronous}, upload by any 
agent may trigger other agents to download the latest data, while our algorithm 
does not mandate this.
On the other hand, many existing algorithms for multi-agent settings 
(e.g., \citet{wang2019distributed}) require all agents to interact with the 
environment in each round, which essentially require full participation of all the agents.   

The determinant-based criterion in Line~\ref{algmain:criterion} has been a 
long-standing design trick in contextual linear bandits that can help address 
the issue of unknown time horizon and reduce the switching cost 
\citep{abbasi2011improved,ruan2021linear}.
For multi-agent bandits, such a criterion has also been used to control the 
communications complexity \citep{wang2019distributed,dubey2020differentially,li2021asynchronous}.
This is because the need for policy switching or communication essentially 
reflects the same fact: enough information has been collected and it is time to 
update the (local) model.
Therefore, achieving low communication complexity and low switching cost are 
unified in our \texttt{FedlinUCB} algorithm in the sense that the communication 
complexity is exactly twice the switching cost.
Furthermore, using lazy update makes our algorithm amenable for analysis, which 
will be clear later in Section~\ref{sec:proof}.
In addition, we leave $\alpha>0$ as a tuning parameter as it controls the 
trade-off between the regret and the communication complexity.

\section{Theoretical Results}\label{sec: theory}
We now present our main result on the theoretical guarantee of Algorithm~\ref{alg: main}.

\begin{theorem}\label{thm:regret}
Under Assumption~\ref{assump: sub gaussian and norm} for Algorithm~\ref{algorithm:1}, 
if we set the confidence radius $\beta= \sqrt{\lambda} 
S + (\sqrt{1+M\alpha}+M\sqrt{2\alpha})
\Big(R\sqrt{d\log \Big(\big(1+TL^2/(\min(\alpha,1)\lambda)\big)/\delta\Big)}
+\sqrt{\lambda}S\Big)$, then with probability at least $1-\delta$, the regret in 
the first $T$ rounds can be upper bounded by
\begin{align*}
\mathrm{Regret}(T)\leq 2dSLM\log(1+TL^2/\lambda)
+2\sqrt{2(1+M\alpha)}\beta\sqrt{2dT\log(1+TL^2/\lambda)}.
\end{align*}
Moreover, the communication complexity and switching cost are both bounded by 
$2\log2 \cdot d(M+1/\alpha) \log(1+TL^2/(\lambda d))$.
\end{theorem}

\begin{remark}
Theorem~\ref{thm:regret} suggests that if we set the parameters $\alpha = 1/M^2$ 
and $\lambda=1/S^2$ in Algorithm~\ref{algorithm:1}, then its regret is 
bounded by $\tilde  O(Rd\sqrt{T})$ and the corresponding communication complexity 
and switching cost are both bounded by $\widetilde O(dM^2)$.
This choice of parameters yields the regret bound and the communication 
complexity presented in Table~\ref{table: comparison with baseline}.
\end{remark}

As a complement, we also provide a lower bound for the communication complexity 
as stated in the following theorem.
See Appendix~\ref{apdx:lower_bound} for the proof.

\begin{theorem}\label{thm:low}
For any algorithm \textbf{Alg} with expected communication complexity less than 
$ O(M/\log (T/M))$, there exist a linear bandit instance with $R=L=S=1$ such that 
for $T\ge Md$, the expected regret for algorithm \textbf{Alg} is at least 
$\Omega(d\sqrt{MT})$.
\end{theorem}

\begin{remark}
Suppose each agent runs the \texttt{OFUL} algorithm \citep{abbasi2011improved} 
separately, then each agent $m\in[M]$ admits an $\tilde  O(d\sqrt{T_m})$ regret, 
where $T_m$ is the number of rounds that agent $m$ participates in. 
Thus the total regret of $M$ agents is upper bounded by $\sum_{m=1}^M  
\tilde  O(\sqrt{T_m}) =  \tilde  O(d\sqrt{MT})$. 
Theorem~\ref{thm:low} implies that for any algorithm \textbf{Alg}, if its 
communication complexity is less than $O(M/\log (T/M))$, then its regret cannot 
be better than naively running $M$ independent \texttt{OFUL} algorithms. 
In other words, Theorem~\ref{thm:low} suggests that in order to perform better 
than the single-agent algorithm through collaboration, an $\Omega(M)$ 
communication complexity is necessary.
\end{remark}

\section{Overview of the Proof}\label{sec:proof}

\begin{wrapfigure}{L}{0.5\textwidth}
\begin{minipage}[t]{0.5\textwidth}
\vspace{-0.3in}
\begin{figure}[H]
    \centering
    \includegraphics[width=0.9\textwidth]{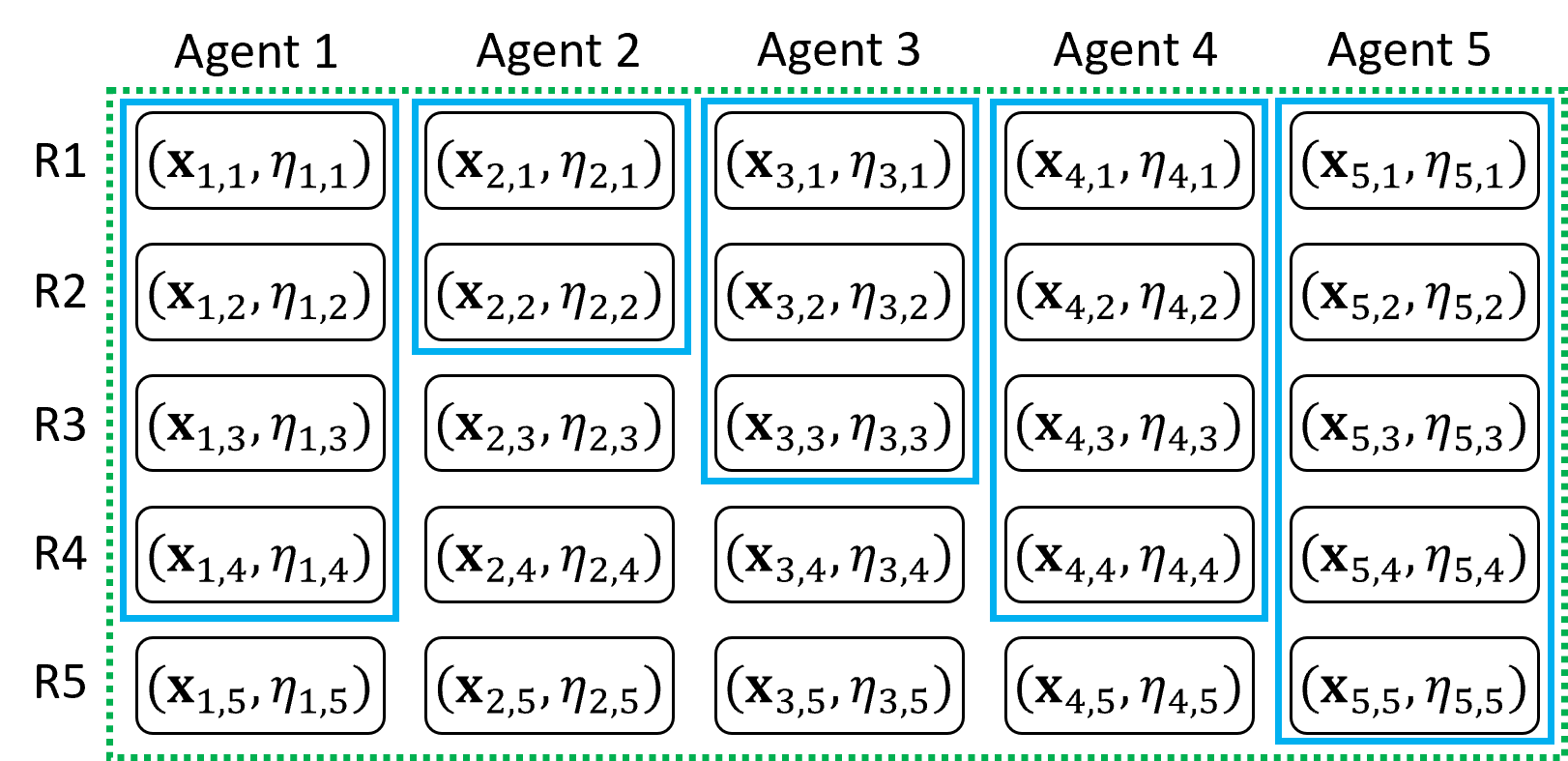}

    \caption{Illustration of ill-defined filtration.}
    \label{fig:filtration}
\end{figure}
\end{minipage}
\end{wrapfigure}

When analyzing the performance of \texttt{FedLinUCB}, we face a unique challenge 
caused by the asynchronous communication, as illustrated in Figure~\ref{fig:filtration}. 
Here $(\xb_{m,t},\eta_{m,t})$ denotes the decision and the noise for agent $m$ 
in its own $t$-th round.
Specifically, in the synchronous setting, the filtration is generated by all the 
data collected by all agents, i.e., $\cF_5 = \sigma\{\xb_{m,t}, \eta_{m,t}\}_{t=1,m=1}^{5,5}$, 
as marked by the green dashed rectangle.
This kind of filtration is well-defined since all agents share their data with 
each other at the end of each round.
In sharp contrast, in our asynchronous setting, the data at the server can be 
generated by an irregular set of data from the agents, as marked by the blue 
rectangles.
Such data pattern can be arbitrary and depends on the data collected in all 
previous rounds, which prevents us from defining a fixed filtration as we can do 
in the synchronous setting.
Since the application of standard martingale concentration inequalities relies on 
the well-defined filtration, they cannot be directly applied to our asynchronous 
setting.

To circumvent the above issue, we need to analyze the concentration property of 
the local data for each agent and then relate it to the concentration of the 
global data, so that we can control the sum of the bonuses and hence the regret.
This requires a careful quantitative comparison of the local and global data 
covariance matrices, which is enabled by our design of determinant-based criterion.
The details will be further explained in Section~\ref{sec:regret_analysis}.
Next, we present the key ingredients in the proof of Theorem~\ref{thm:regret}.

\begin{remark}
Recall the notations in Table~\ref{tab:notation},
where the values of those matrices and vectors might change within each round.
To eliminate the possible confusion, from now on we follow the convention that
all matrices and vectors in the analysis correspond to their values at the end 
of each round in Algorithm~\ref{alg: main}.
\end{remark}

\subsection{Analysis for communication complexity and switching cost}
We first analyze the communication complexity and switching cost of Algorithm~\ref{alg: main}.
For each $i \geq 0$, we define
\begin{align}\label{eq:epoch}
    \tau_i = \min \{t \in [T] \mid \det(\bSigma_t^\ser) \geq 2^i\lambda^d\}.    
\end{align}
We divide the set of all rounds into epochs $\{\tau_i, \tau_{i}+1, \ldots, \tau_{i+1}-1\}$ 
for each $i \geq 0$.
Then the communication complexity \emph{within each epoch} can be bounded using 
the following lemma.

\begin{lemma}\label{lem:communication_cost}
Under the setting of Theorem~\ref{thm:regret}, for each epoch from round 
$\tau_i$ to round $\tau_{i+1}-1$, the number of communications in this 
epoch is upper bounded by $2(M+1/\alpha)$.
\end{lemma}
\begin{proof}[Proof of Theorem~\ref{thm:regret}: communication complexity and switching cost.]
It suffices to bound the number of epochs.
By Assumption~\ref{assump: sub gaussian and norm}, we have $\|\xb_t\|_2 \leq L$ 
for all $t \in [T]$.
Since $\bSigma_T^\ser$ is positive definite, by the inequality of arithmetic and 
geometric means, we have
\begin{align*}
    \det(\bSigma_T^\ser) &\leq \bigg(\frac{\tr(\bSigma_T^\ser)}{d}\bigg)^d 
    \leq \bigg(\frac{1}{d}\tr\bigg(\lambda \Ib + \sum_{t=1}^T \xb_t\xb_t^\top\bigg)\bigg)^d\\
    &= \bigg(\lambda + \frac{1}{d} \sum_{t=1}^T \|\xb_t\|_2^2\bigg)^d
    \leq \lambda^d \bigg(1 + \frac{TL^2}{\lambda d}\bigg)^d.
\end{align*}
Then recalling the definition of epochs based on~\eqref{eq:epoch}, we have 
\begin{align*}
    \max\{i\geq 0 \mid \tau_i \neq \emptyset\} = \log_2 \frac{\det(\bSigma_T^\ser)}{\lambda^d}
    \leq \log2 \cdot d \log\bigg(1 + \frac{TL^2}{\lambda d}\bigg).
\end{align*} 
Therefore, the total number of epochs is bounded by 
$\log 2 \cdot d \log(1 + TL^2/(\lambda d))$.
Now applying Lemma~\ref{lem:communication_cost}, the total communication 
complexity is bounded by $2\log2 \cdot d(M+1/\alpha) \log(1+TL^2/(\lambda d))$.
Note that in Algorithm~\ref{alg: main}, each agent only switch its policy after 
communicating with the server, so the switching cost is exactly equal 
to half of the communication complexity.
This finishes the proof for the claim on communication complexity and switching 
cost in Theorem~\ref{thm:regret}.
\end{proof}

\subsection{Analysis for regret upper bound}\label{sec:regret_analysis}
The regret analysis for Theorem~\ref{thm:regret} is much more involved, and it 
relies on a series of intermediate lemmas establishing the concentration.

\paragraph{Total information.}
We define the following auxiliary matrices and vectors that contain all the 
information up to round $t$:
\begin{align}
    \bSigma^\all_t = \lambda \Ib + \sum_{i=1}^t \xb_i \xb^{\top}_i, \qquad 
    \bbb^\all_t = \sum_{i=1}^t r_i \xb_i, \qquad
    \ub^\all_t = \sum_{i=1}^t \eta_i \xb_i,
\end{align} 
where $\eta_i \coloneqq r_i - \langle \xb_i, \btheta^* \rangle$ is a 
$R$-sub-Gaussian noise by Assumption~\ref{assump: sub gaussian and norm}.
In our setting, $\bSigma_t^\all, \bbb_t^\all, \ub_t^\all$ are not 
accessible by the agents due to asynchronous communication, and we only use them 
to facilitate the analysis.
With this notation, we can further define the following omnipotent estimate:
\begin{align}\label{eq:global_estimate}
    \hat{\btheta}_t^\all = (\bSigma_t^\all)^{-1} \bbb_t^\all. 
\end{align}
As a direct application of the self-normalized martingale concentration 
inequality~\citep{abbasi2011improved}, we have the following global confidence 
bound due to the concentration of $\bSigma_t^\all$ and $\bbb_t^\all$.

\begin{lemma}[Global confidence bound; Theorem 2, \citealt{abbasi2011improved}]
\label{lemma:global-concentration}
With probability at least $1-\delta$, for each round $t \in [T]$, the estimate 
$\hat{\btheta}_t^\all$ in \eqref{eq:global_estimate} satisfies
\begin{align*}
\| \hat{\btheta}_t^\all-\btheta^*\|_{\bSigma_t^\all}
\leq R\sqrt{d\log\big((1+TL^2/\lambda)/\delta\big)}+\sqrt{\lambda}S.
\end{align*}
\end{lemma}

\paragraph{Per-agent information.}
Next, for each agent $m \in [M]$, we denote the rounds when agent $m$ 
communicate with the server (i.e., upload and download data) as 
$\{t_{m,1}, t_{m,2}, ...\}$. 
For simplicity, \emph{at the end of round $t$}, we denote by $N_m(t)$ the last 
round when agent $m$ communicated with the server (so if agent $m$
communicated with the server in round $t$, then $N_m(t) = t$). 
With this notation, for each round $t$ and agent $m \in [M]$, the data that has 
been uploaded by agent $m$ is then\footnote{Strictly speaking, the uploaded data 
only consists of $\bSigma_{m,t}^\up$ and $\bbb_{m,t}^\up$, and here we introduce 
$\ub_{m,t}^\up$ and $\ub_{m,t}^\loc$ solely for the purpose of analysis.}
\begin{align*}
    \bSigma_{m,t}^\up=\sum_{j=1,m_j=m}^{N_m(t)}
    \xb_j\xb^{\top}_j, \qquad \ub_{m,t}^\up = \sum_{j=1, m_j=m}^{N_m(t)} 
    \xb_j \eta_j.
\end{align*}
Correspondingly, the local data that has not been uploaded to the server is
\begin{align*}
    \bSigma_{m,t}^\loc = \sum_{j=N_m(t)+1, m_j=m}^t \xb_j \xb^{\top}_j, \qquad 
    \ub_{m,t}^\loc = \sum_{j=N_m(t)+1, m_j=m}^t \xb_j \eta_j.
\end{align*}

Again, applying the self-normalized martingale concentration \citep{abbasi2011improved} 
together with a union bound, we can get the per-agent local concentration.

\begin{lemma}[Local concentration]\label{lemma:local-concentration}
Under the setting of Theorem~\ref{thm:regret}, with probability at least $1-\delta$, 
for each round $t\in[T]$ and each agent $m\in [M]$, it holds that
\begin{align*}
    \Big\|\left(\alpha\lambda \Ib + \bSigma_{m,t+1}^\loc\right)^{-1} 
    \ub_{m,t}^\loc\Big\|_{\alpha\lambda \Ib + \bSigma_{m,t}^\loc}
    \leq R\sqrt{d\log \Big(\big(1+TL^2/(\alpha\lambda)\big)/\delta\Big)}+\sqrt{\lambda}S.
\end{align*}
\end{lemma}

Moreover, based on our determinant-based communication criterion, we have the 
following lemma that describes the quantitative relationship among the local 
data, uploaded data and global data.
\begin{lemma}[Covariance comparison]\label{lemma:covariance_comparison}
Under the setting of Theorem~\ref{thm:regret}, it holds that 
\begin{align}\label{new1}
    \lambda \Ib + \sum_{m'=1}^M \bSigma_{m',t}^\up \succeq 
    \frac{1}{\alpha}\bSigma_{m,t}^\loc
\end{align}
for each agent $m \in [M]$.
Moreover, for any $1\leq t_1<t_2\leq T$, if agent $m$ is the only active agent from round $t_1$ to $t_2-1$ and agent $m$ only communicates with the server at round $t_1$, then for all $t_1 + 1 \leq t \leq t_2$, it further holds that
\begin{align}\label{new2}
    \bSigma_{m,t} \succeq \frac{1}{1+M\alpha} \bSigma_t^\all.
\end{align}
\end{lemma}
Combining the above results, we obtain the local confidence bound which then 
leads to the per-round regret in each round, as summarized in the following lemma.
\begin{lemma}[Local confidence bound and per-round regret]\label{lemma:local_confidence_bound}
Under the setting of Theorem~\ref{thm:regret}, with probability at least $1 - \delta$, for each $t \in [T]$, the estimate $\hat{\btheta}_{m,t+1}$ satisfies that $\|\btheta^* - \hat{\btheta}_{m,t+1}\|_{\bSigma_{m,t+1}} \leq \beta$ for all $m\in [M]$.
Consequently, for each round $t\in[T]$, the regret in round $t$ satisfies
\begin{align*}
\Delta_t= \max_{\xb \in \cD_t}\la \btheta^*, \xb\ra - \la \btheta^*, \xb_t\ra
\leq 2\beta\sqrt{\xb_t^{\top}\bSigma_{m_t,t}^{-1}\xb_t}.
\end{align*}
\end{lemma}

Now, we are ready to prove the regret bound in Theorem~\ref{thm:regret}.
\begin{proof}[Proof of Theorem~\ref{thm:regret}: regret]
Firstly, %recall that the regret can decomposed as the sum of per-round regret:
according to Lemma \ref{lemma:local_confidence_bound}, the regret in the
first $T$ round can be decomposed and upper bounded by
\begin{align}\label{eq:regret}
    \regret(T) &= \sum_{t=1}^T \big( \max_{\xb \in \cD_t} \langle \btheta^*, \xb\rangle - \langle \btheta^*, \xb_t\rangle\big)
    \leq \sum_{t=1}^T 2\beta \|\xb_t\|_{\bSigma_{m_t,t}^{-1}}.
\end{align}
Now, we only need to control the summation of bonus term 
$2\beta \|\xb_{t}\|_{\bSigma_{m_t,t}^{-1}}$, and we focus on the agent-action 
sequence $\{(m_t,\xb_t)\}_{t=1}^T$.
Note that if agent $m$ communicates with the server at round $t_1$ and $t_2$, 
then the order of actions between round $t_1$ and $t_2$ will not effect the 
covariance matrix of agent $m$. 
In addition, since agent $m$ does not upload new data between round $t_1$ and $t_2$, 
the order of actions from agent $m$ will not affect other agents' covariance matrix. 
Thus, without affecting the covariance matrix and the corresponding bonus, we can always reorder the sequence of active agents such that 
each agent communicates with the server and stays active until the next agent 
kicks in to communicate with the server.
Such reordering is valid according to the communication protocol as each agent 
has only local updates between communications with the server.

More specifically, we assume that the sequence of rounds that the active agent 
communicates with server is $0 = t_0 < t_1 < t_2 < \cdots < t_N = T+1$\footnote{
There is no communication happening at $t_0$ or $t_N$, but we include them for 
notational convenience.}, and from round $t_i+1$ to $t_{i+1}-1$ there is only 
one agent active, that is, $m_{t_i} = m_{t_i+1} = \cdots = m_{t_{i+1} - 1}$.
Therefore, the regret upper bound in \eqref{eq:regret} can be refined as
\begin{align}\label{new}
    \regret(T) 
    \leq\sum_{i=0}^{N-1}\sum_{t=t_i+1}^{t_{i+1}-1} 2 \beta \|\xb_t\|_{\bSigma_{m_t,t}^{-1}}+\sum_{i=1}^{N-1}\min\bigg\{\max_{\xb \in \cD_{t_i}} \langle\btheta^*, \xb\rangle - \langle\btheta^*, \xb_{t_i}\rangle,2 \beta \|\xb_{t_i}\|_{\bSigma_{m_{t_i},t_i}^{-1}}\bigg\}.
\end{align}
Applying \eqref{new2} in Lemma~\ref{lemma:covariance_comparison}, the bonus term for each 
round $t_i < t < t_{i+1}$ can be controlled by
\begin{align*}
   2\beta \|\xb_{t}\|_{\bSigma_{m_t,t}^{-1}}\leq 2 \beta \sqrt{1+M\alpha}
   \|\xb_{t}\|_{(\bSigma_t^\all)^{-1}}.
\end{align*}
Collecting these terms, we obtain
\begin{align}\label{eq:0-000}
    \sum_{i=0}^{N-1}\sum_{t=t_i+1}^{t_{i+1}-1} 2 \beta \|\xb_t\|_{\bSigma_{m_t,t}^{-1}} \leq \sum_{i=0}^{N-1} \sum_{t=t_i+1}^{t_{i+1}-1} 2\beta\sqrt{1+M\alpha} \|\xb_t\|_{(\bSigma_t^\all)^{-1}}.
\end{align}
It remains to control the bonus terms for rounds $\{t_i\}_{i=1}^N$. For a more refined analysis of the rounds $\{t_i\}_{i=1}^N$, we define
\begin{align*}
    T_i = \min\big\{t\in[T] \mid \det(\bSigma_t^\all) \geq 2^i\lambda^d\big\},
\end{align*} 
and let $N'$ be the largest integer such that $T_{N'}$ is not empty.
For each time interval from $T_i$ to $T_{i+1}$ and each agent $m \in [M]$, suppose agent $m$ communicates with 
the server more than once, where the communications occur at rounds $T_{i,1}, T_{i,2}, \ldots, T_{i,k} \in [T_i, T_{i+1})$. 
Then for each $j = 2, \ldots, k$, since agent $m$ is active at rounds $T_{i, j-1}$ and $T_{i, j}$, applying \eqref{new2} in Lemma~\ref{lemma:covariance_comparison}, the bonus term for round $T_{i,j}$ can be controlled by
\begin{align*}
2\beta \|\xb_{T_{i,j}}\|_{\bSigma_{m,T_{i,j}}^{-1}}&\leq 2\beta \|\xb_{T_{i,j}}\|_{\bSigma_{m,T_{i,j-1}+1}^{-1}} \leq 2\beta\sqrt{1+M\alpha}  \|\xb_{T_{i,j}}\|_{(\bSigma_{T_{i,j-1}+1}^{\all})^{-1}}.
\end{align*}
Since $\det(\bSigma_{T_{i+1}-1}^\all) / \det(\bSigma_{T_{i,j-1}+1}^\all) \leq 2^{i+1} \lambda^d / (2^i \lambda^d) = 2$ by the definition of $T_i$, it further follows from Lemma~\ref{lemma:det} that
\begin{align}
2\beta \|\xb_{T_{i,j}}\|_{\bSigma_{m,T_{i,j}}^{-1}} \leq 2\beta\sqrt{2(1+M\alpha)} \|\xb_{T_{i,j}}\|_{(\bSigma_{T_{i+1}-1}^{\all})^{-1}}
\leq 2\beta\sqrt{2(1+M\alpha)} \|\xb_{T_{i,j}}\|_{(\bSigma_{T_{i,j}}^{\all})^{-1}},\label{eq:0-001}
\end{align}
where the second inequality is due to 
the fact that $\bSigma_{T_{i+1}-1}^{\all} \succeq \bSigma_{T_{i,j}}^{\all}$.
Moreover, for round $T_{i,1}$ where the first communication occurs, we can trivially bound the per-round regret $\big(\max_{\xb \in \cD_{t_i}} \langle\btheta^*, \xb\rangle - \langle\btheta^*, \xb_{t_i}\rangle\big)$  in that round by $2SL$, instead of using the bonus.
Then combining this with \eqref{eq:0-001} for all time intervals and all agents, we obtain
\begin{align*}
    \sum_{i=1}^{N-1}\min\bigg\{\max_{\xb \in \cD_{t_i}} \langle\btheta^*, \xb\rangle - \langle\btheta^*, \xb_{t_i}\rangle,2 \beta \|\xb_{t_i}\|_{\bSigma_{m_{t_i},t_i}^{-1}}\bigg\}
    \leq 2SLMN' + \sum_{i=1}^{N-1} 2\beta\sqrt{2(1+M\alpha)} \|\xb_{t_i}\|_{(\bSigma_{t_i}^\all)^{-1}}.
\end{align*}
To bound $N'$, note that the norm of each action $\xb$ satisfies $\|\xb_t\|_2\leq L$ by Assumption~\ref{assump: sub gaussian and norm}, and thus
\begin{align*}
    \det(\bSigma_T^{\all})\leq (\lambda+TL^2)^d,
\end{align*}
which implies that $N'$ is at most 
$d\log(1+TL^2/\lambda)$. 
Therefore, we further have 
\begin{align}\label{eq:0-002}
    \sum_{i=1}^{N-1}\min\bigg\{\max_{\xb \in \cD_{t_i}} \langle\btheta^*, \xb\rangle - \langle\btheta^*, \xb_{t_i}\rangle,2 \beta \|\xb_{t_i}\|_{\bSigma_{m_{t_i},t_i}^{-1}}\bigg\} &\leq 2dSLM\log(1+TL^2/\lambda)\notag\\
    &\qquad+ \sum_{i=1}^{N-1} 2\beta\sqrt{2(1+M\alpha)} \|\xb_{t_i}\|_{(\bSigma_{t_i}^\all)^{-1}}.
\end{align}
Finally, substituting \eqref{eq:0-000} and \eqref{eq:0-002} into \eqref{new}, we obtain
\begin{align*}
    \regret(T)&\leq 2dSLM\log(1+TL^2/\lambda)+\sum_{t=1}^T 2\sqrt{2(1+M\alpha)}\beta \|\xb_t\|_{(\bSigma_t^{\all})^{-1}}\notag\\
    &\leq 2dSLM\log(1+TL^2/\lambda)+ 2\sqrt{2(1+M\alpha)}\beta\sqrt{2dT\log(1+TL^2/\lambda)},
\end{align*}
where the last inequality follows from a standard elliptical potential 
argument~\citep{abbasi2011improved}.
\end{proof}

\section{Conclusion and Future Work}
In this work, we study federated contextual linear bandit problem with fully 
asynchronous communication. 
We propose a simple and provably efficient algorithm named \texttt{FedLinUCB}. 
We prove that \texttt{FedLinUCB} obtains a near-optimal regret of order 
$\tilde O(d\sqrt{T})$ with $\tilde O(dM^2)$ communication complexity. 
We also prove a lower bound on the communication complexity, which suggests 
that an $\Omega(M)$ communication complexity is necessary for achieving a 
near-optimal regret.
There still exists an $O(dM)$ gap between the upper and lower bounds for the 
communication complexity and we leave it as a future work to close this gap.
Another important direction for future work is to study federated linear bandits 
with a decentralized communication network without a central server (i.e., P2P networks).

%\newpage
\appendix

\section{Further Discussions}\label{sec: further discussion}
Here we provide further discussions on our results. 
We present a detailed comparison with \citet{li2021asynchronous} in 
Appendix~\ref{sec: comparison}, where we first discuss the difference in the 
algorithmic design, and then elaborate on 
the concentration issue under the asynchronous setting with a simple example. 
In Appendix~\ref{sec:alternative form}, we give an alternative form of our
algorithm, where we rewrite Algorithm~\ref{alg: main} in an `episodic' fashion. 
The purpose is to make it easier for readers to compare our algorithm with
existing algorithms for federated linear bandits that are usually expressed in 
the `episodic' form.

\subsection{Comparison with \citet{li2021asynchronous}}\label{sec: comparison}

\paragraph{Difference in algorithmic design.}
The Async-LinUCB algorithm proposed by \citet{li2021asynchronous} is not fully 
asynchronous since in their algorithm, if some agent uploads data to the 
server, the server will decide if each of the $M$ agents needs to download the data. 
If the server decides that an agent needs to download the data, this agent has 
to first download the data from the server and then update its local policy 
before further interaction with the environment (i.e., taking the next action). 
In other words, if an agent is offline when the server requests a download, the 
agent cannot take any further action until it goes online and completes the 
required download and local model update. 
In contrast, under the communication protocol in our Algorithm~\ref{alg: main}, 
any offline agent can still take action until the trigger of the upload protocol.
It is evident that their asynchronous communication protocol is very restricted. 

\paragraph{Concentration issue.}
Next, we discuss the concentration issue, and we first illustrate the problem 
using a multi-arm bandit instance.
Unlike the synchronous case, the reward estimator based on the server-end data 
can be biased in asynchronous federated linear bandits. 
To see so, let us consider the following simple example: The decision set 
contains two arms, $A$ and $B$, and suppose for pulling arm $A$, the agent 
receives a reward equal to either $1$ or $-1$ with equal probability.
We assume that there are $M$ agents, and each agent is active for two consecutive rounds. 
For each agent $m \in [M]$, if the agent has selected the arm $A$ in the first 
round, then the agent will select again the arm $A$ in the second round only if 
the agent receives a reward of $1$ when pulling arm $A$ in the first round. 
In this case, it is easy to show that with probability $0.5$, an agent selects 
arm $A$ one time with reward $-1$, and with probability $0.25$, an agent selects 
arm $A$ twice with total reward $2$. 
Similarly, with probability $0.25$, an agent selects arm $A$ twice with a total 
reward of $0$. 

In the synchronous setting, all agent will upload their local data to the server
at the end of each round. 
Thus, taking an average for all data at the server, the expected reward of arm 
$A$ is $0$, which equals the actual expected reward of arm $A$.
However, in the asynchronous setting, things become more complicated. 
Suppose that for each agent, only selecting arm $A$ twice will trigger the upload 
protocal. 
Then after two active rounds, an agent will upload its data to the server if and 
only if the agent receives reward 1 in the first round. 
Thus among the agents that upload the data, half of them receive a total reward of $2$ 
and the other half receive a total reward of $0$.
In this case, taking an average for all data at the server, the expected reward 
of arm $A$ is $0.5$, which is a biased estimator compared with the actual 
expected reward. 

Indeed, the above issue could happen in federated linear bandits with the 
Async-LinUCB algorithm~\citep{li2021asynchronous}.
Specifically, let us consider a linear bandit instance with dimension $d=2$, 
and we assume that arm $A$ has context vector $\xb_{A}=(3,0)^\top$, arm $B$ has 
context vector $\xb_{B}=(0,1/\sqrt{10})^\top$, the true model is $\btheta^* = \zero$, the noise $\eta$ is a Rademacher random variable, and the parameter $\lambda$ is set to be 1. Therefore, the rewards for both arm $A$ and $B$ equal to $1$ or $-1$ with $0.5$ probability.
In this case, based on the principle of optimism in the face of uncertainty, 
at the beginning, the optimistic estimators for the two arms $A,B$ are $3\beta$ 
and $\beta/\sqrt{10}$ respectively. 
Thus, all agents will always choose arm $A$ in the first round, 
so $\xb_1 = \xb_A$. 
After choosing arm $A$ at the first round, the optimistic estimator for the two 
arms $A,B$ in each agent's second round will be $9r/10 + 3\beta/\sqrt{10}$ and 
$\beta/\sqrt{10}$ respectively, where $r$ is the reward received in the first round. 
Therefore, with confidence radius $\beta< 1$, each agent will select the arm $A$ (i.e., $\xb_2 = \xb_A$) in the 
second round only if the agent receives a reward of $r=1$ in the first round. 
Finally, only choosing arm $A$ twice will increase the determinant of the 
covariance matrix enough to trigger the upload protocol 
(e.g., $\det(\lambda \Ib+\xb_1\xb_1^{\top}+\xb_2\xb_2^{\top})/\det(\lambda \Ib)\ge 19$). 

As demonstrated above, in the asynchronous setting, the reward estimator based on 
the server-end data can be biased, which leads to the issue that previous 
concentration results (e.g., \citet{abbasi2011improved}) cannot be directly used 
for the server's data.
This is why we need a more dedicated analysis to control this biased error 
(see Lemma~\ref{lemma:local_confidence_bound} for more details).

\subsection{An Alternative Form of Algorithm \ref{alg: main}}\label{sec:alternative form}
\begin{algorithm}[t]
	\caption{Federated linear UCB (Alternative)}
	\label{alg:alternative}
	\begin{algorithmic}[1]
	\STATE Initialize $\bSigma_{m, 1} = \bSigma^\ser_1 = \lambda \Ib$, $\hat\btheta_{m,1}=0$, $\bbb_{m,0}^\loc =0$ and 
    $\bSigma^\loc_{m,0}=0$ for all $m \in [M]$
	\FOR{$k = 1, 2, \dots, K$}\alglinelabel{algline: episode}
	\STATE Participation set $P_k \subseteq [M]$ of arbitrary order\alglinelabel{algline: participation set}
	\FOR{each active agent $m \in P_k$}\alglinelabel{algline: for agents in set}
	\STATE Receive $D_{m, k}$ from the environment\alglinelabel{algline: receive decision set}
	\STATE Select $\xb_{m, k} \leftarrow \argmax_{\xb\in D_{m, k}} \langle 
    \hat\btheta_{m,k}, \xb_k\rangle + \beta\|\xb\|_{\bSigma_{m,k}^{-1}}$ 
    \alglinelabel{algalt:optimistic_decision} {\color{blue}\hfill \texttt{/* Optimistic decision */}}
	\STATE Receive $r_{m,k}$ from environment
	\STATE $\bSigma^\loc_{m, k}\leftarrow \bSigma^\loc_{m, k-1} 
    + \xb_{m, k} \xb_{m, k}^\top$, \quad $\bbb^\loc_{m, k} \leftarrow 
    \bbb^\loc_{m, k-1} + r_{m, k} \xb_{m, k}$ \hfill{\color{blue}\texttt{/*Local update*/}}
    \alglinelabel{algalt:local_update}
	\IF{$\det(\bSigma_{m, k} + \bSigma^\loc_{m, k}) > (1 + \alpha)
    \det(\bSigma_{m,k})$}\alglinelabel{alg1:criterion}
	\STATE Agent $m$ sends $\bSigma^\loc_{m,k}$ and $\bbb^\loc_{m,k}$ to server
    \hfill{\color{blue}\texttt{/* Upload */}} \alglinelabel{algalt:upload}
	\STATE $\bSigma^\ser_k \leftarrow \bSigma^\ser_k + \bSigma^\loc_{m,k}, \quad \bbb^\ser_k \leftarrow \bbb^\ser_k  + \bbb^\loc_{m,k}$ 
    \hfill{\color{blue}\texttt{/* Global update */}} \alglinelabel{algalt:server_update}
    \STATE $\bSigma^\loc_{m, k} \leftarrow 0, \quad \bbb^\loc_{m, k}\leftarrow 0$
    \alglinelabel{algalt:reset_local}
    \STATE Server sends $\bSigma^\ser_k$ and $\bbb^\ser_k$ back to agent $m$ 
    \hfill{\color{blue}\texttt{/* Download */}} \alglinelabel{algalt:download1}
	\STATE $\bSigma_{m, k+1} \leftarrow \bSigma^\ser_k, \quad \bbb_{m, k} \leftarrow 
    \bbb^\ser_k$ \alglinelabel{algalt:download2}
	\STATE $\hat\btheta_{m, k+1} \leftarrow \bSigma_{m, k+1}^{-1} \bbb_{m, k+1}$ 
    \hfill{\color{blue}\texttt{/* Compute estimate */}} \alglinelabel{algalt:estimate}
    \ELSE 
    \STATE $\bSigma_{m,k+1} \leftarrow \bSigma_{m,k}$, \quad $\bbb_{m,k+1} \leftarrow \bbb_{m,k}$, \quad $\hat\btheta_{m,k+1} \leftarrow \hat\btheta_{m,k}$
    \ENDIF\alglinelabel{algline: endif for alternative}
	\ENDFOR
	\FOR{other inactive agents $m \in [M] \setminus P_k$}
    \STATE $\bSigma_{m,k+1} \leftarrow \bSigma_{m,k}, \quad \bbb_{m,k+1} \leftarrow \bbb_{m,k}, \quad \hat\btheta_{m,k+1} \leftarrow \hat\btheta_{m,k}$
	\ENDFOR
	\ENDFOR
	\end{algorithmic}
\end{algorithm}

We introduce an alternative form of Algorithm \ref{alg: main}, which is displayed 
in Algorithm \ref{alg:alternative}. 
Algorithm~\ref{alg:alternative} can be viewed as the episodic\footnote{Here `episode' 
means a collection of every agent's interaction with the environment for one round, 
which is different from the usual term in online learning that refers to a 
sequential interaction lasting for a certain time horizon.
We only use this term to differentiate Algorithm~\ref{alg:alternative} from 
Algorithm~\ref{alg: main}.} version of Algorithm~\ref{alg: main}, and its form 
aligns with those of the existing algorithms for federated linear bandits 
\citep{wang2019distributed,dubey2020differentially,huang2021federated,li2021asynchronous}.

Specifically, in Algorithm~\ref{alg:alternative}, for each round (episode) $k \in [K]$, 
the set of active agents is given by $P_k$, where the order of agents in $P_k$ 
can be arbitrary (Line~\ref{algline: participation set}). 
Then the agents in the set $P_k$ participate according to the prefixed order 
(Line~\ref{algline: for agents in set}). 
The operations in the inner loop of Algorithm~\ref{alg:alternative} 
(i.e., decision rule, upload/download, local/global update, and model estimates) 
are all identical to those in Algorithm~\ref{alg: main}. 
Therefore, Algorithm~\ref{alg:alternative} is indeed equivalent to 
Algorithm~\ref{alg: main} up to relabeling of the participation of the agents, and 
hence all the theoretical results for Algorithm \ref{alg: main} also hold for Algorithm~\ref{alg:alternative}.

\section{Missing Proofs in Section~\ref{sec:proof}}
Here we present the proof of the results in Section~\ref{sec:proof}.

\subsection{Communication complexity within each epoch}
We first present the proof for the bound on the communication complexity within 
each epoch given in Lemma~\ref{lem:communication_cost}.
\begin{proof}[Proof of Lemma \ref{lem:communication_cost}]
For each agent $m \in [M]$, let $n_m$ be the number of communications agent $m$ 
has made during this epoch, and we denote the communication rounds as 
$t_1,\ldots,t_{n_m}$ for simplicity. 
Now we consider the data uploaded to the server, and it can be denoted by the 
value of covariance matrix $\bSigma_{m,t_j}^{\loc}$ before communicating with the server. 
For each $j=2, \ldots, n_m$, according to the determinant-based 
criterion (Line \ref{alg1:criterion}) in Algorithm \ref{algorithm:1}, we have
\begin{align*}
\det(\bSigma_{m,t_j}+\bSigma_{m,t_j}^{\loc})-\det(\bSigma_{m,t_j})
> \alpha \cdot \det(\bSigma_{m,t_j}),
\end{align*}
which further implies that
\begin{align}
    \alpha \cdot \det(\bSigma_{T_i}^{\ser})
    < \det(\bSigma_{T_i}^{\ser}+\bSigma_{m,t_j}^{\loc})-\det(\bSigma_{T_i}^{\ser}),\label{eq:004}
\end{align}
where the inequality holds due to Lemma \ref{lem:det2} together with the fact 
that the communication in round $t_1$ updates the covariance matrix so that 
$\bSigma_{m,t_j}\succeq\bSigma_{T_i}^{\ser}$.
In addition, we define the sequence of all communications from $T_i$ to $T_{i+1}-1$ 
as $t'_1,\ldots,t'_L$.  
For each round $t'_j$, if the agent $m_{t'_j}$ have already communicated with 
the server earlier in this epoch, we have
\begin{align}
    \det(\bSigma_{t'_j}^{\ser})-\det(\bSigma_{t'_{j-1}}^{\ser})&=\det(\bSigma_{t'_{j-1}}^{\ser}+\bSigma_{m_{t'_j},t'_j}^{\loc})-\det(\bSigma_{t'_{j-1}}^{\ser})\notag\\
    &\ge \det(\bSigma_{T_i}^{\ser}+\bSigma_{m,t_j}^{\loc})-\det(\bSigma_{T_i}^{\ser}) \notag\\
    &> \alpha \cdot \det(\bSigma_{T_i}^{\ser}),\label{eq:005}
\end{align}
where the first inequality holds due to Lemma~\ref{lem:det1} together with the 
fact that $\bSigma_{t'_{j-1}}^{\ser} \succeq \bSigma_{T_i}^{\ser}$, and the 
second inequality follows from \eqref{eq:004}.
Now, taking the sum of \eqref{eq:005} over all round $t'_j$, we obtain
\begin{align}
\det(\bSigma_{T_{i+1}-1}^\ser)-\det(\bSigma_{T_{i}}^\ser)
=\sum_{1\leq j\leq L} \det(\bSigma_{t'_j}^{\ser})-\det(\bSigma_{t'_{j-1}}^{\ser})
\ge \sum_{m=1}^M(n_m-1)\alpha \cdot \det(\bSigma_{T_i}^\ser).\notag
\end{align}
Since $\det(\bSigma_{\ser,T_{i+1}-1})\leq 2 \det(\bSigma_{\ser,T_i})$, we further have
\begin{align*}
    \sum_{j\in M} n_j\leq M+1/\alpha. 
\end{align*}
Each communication between one agent and the server includes one upload and one download, so the communication complexity within one epoch is bounded by $2(M + 1 / \alpha)$.
This finishes the proof.
\end{proof}

\subsection{Proof for the covariance comparison}
Next, we prove the comparison between the covariance matrices given in Lemma~\ref{lemma:covariance_comparison}.

\begin{proof}[Proof of Lemma \ref{lemma:covariance_comparison}]
Fix any round $t \in [T]$.
Let $t_1 \leq t$ be the last round such that agent $m$ is active at round $t_1$.
If agent $m$ communicated with the server at this round, then we have
\begin{align*}
    \lambda \Ib + \sum_{m'=1}^M \bSigma_{m',t}^\up \succeq  \mathbf{0}
    = \frac{1}{\alpha}\bSigma_{m,t}^\loc.
\end{align*}
Otherwise, according to determinant-based criterion (Line \ref{alg1:criterion})
in Algorithm \ref{algorithm:1}, at the end of each round $t_1$,
we have
\begin{align*}
    \det\left(\bSigma_{m,t_1} + \bSigma_{m,t_1}^\loc\right) \leq 
    \left(1 + \alpha\right) \det(\bSigma_{m,t_1}).
\end{align*}
By Lemma \ref{lemma:det}, for any non-zero vector $\xb \in \RR^d$, we have
\begin{align*}
    \frac{\xb^{\top} (\bSigma_{m,t_1} + \bSigma_{m,t_1}^\loc) \xb}{\xb^{\top} 
    \bSigma_{m,t_1} \xb} \leq \frac{\det(\bSigma_{m,t_1} + \bSigma_{m,t_1}^\loc)}{\det 
    (\bSigma_{m,t_1})} \leq 1 + \alpha.
\end{align*}
Rearranging the above yields $\xb^{\top} \bSigma_{m,t_1}^\loc \xb \leq \alpha\xb^{\top}
\bSigma_{m,t_1} \xb$, which then implies that 
\begin{align*}
    \bSigma_{m,t_1} \succeq  \frac{1}{\alpha} \bSigma_{m,t_1}^\loc
\end{align*}
Note that $\bSigma_{m,t_1}$ is the downloaded covariance matrix from last 
communication before round $t_1$, so it must satisfy $\bSigma_{m,t_1} \preceq \bSigma_{t_1}^\ser$. 
Therefore, we have
\begin{align*}
    \lambda \Ib + \sum_{m'=1}^M \bSigma_{m',t_1}^\up=  \bSigma_{t_1}^\ser \succeq \bSigma_{m,t_1} 
    \succeq  \frac{1}{\alpha} \bSigma_{m,t_1}^\loc.
\end{align*} 
Now, for round $t$, since agent $m$ is inactive from round $t_1$ to $t$, then we have
\begin{align*}
    \lambda \Ib + \sum_{m'=1}^M \bSigma_{m',t}^\up
    \succeq  \lambda \Ib + \sum_{m'=1}^M \bSigma_{m',t_1}^\up
    \succeq \frac{1}{\alpha} \bSigma_{m,t_1}^\loc
    = \frac{1}{\alpha} \bSigma_{m,t}^\loc,
\end{align*}
which yields the first claim in Lemma~\ref{lemma:covariance_comparison}.

Next, suppose agent $m$ is the only active agent from round $t_1$ to $t_2-1$ and agent $m$ only communicates with the server at round $t_1$.
Further average the above inequality over all agents $m \in [M]$, and we get
\begin{align}
\lambda \Ib + \sum_{m'=1}^M \bSigma_{m',t}^\up \succeq
\frac{1}{M\alpha}\sum_{m'=1}^M\bSigma_{m',t}^\loc,\label{eq:0-2}
\end{align}
which implies that for $t_1+1\leq t \leq t_2-1$, we have
\begin{align*}
\bSigma_{m,t} &= \lambda \Ib + \sum_{m'=1}^M \bSigma_{m',t_1}^\up 
= \lambda \Ib + \sum_{m'=1}^M \bSigma_{m',t}^\up \notag\\
&\succeq \frac{1}{1+M\alpha} \bigg(\lambda \Ib + \sum_{m'=1}^M \bSigma_{m',t}^\up 
+ \sum_{m'=1}^M\bSigma_{m',t}^\loc\bigg)
= \frac{1}{1+M\alpha} \bSigma_t^\all,
\end{align*}
where the second equation holds due to the fact that no agent communicate with 
server from round $t_1+1$ to $t_2-1$, and the inequality follows from \eqref{eq:0-2}. 
This yields the second claim in Lemma~\ref{lemma:covariance_comparison} and finishes the proof.
\end{proof}

\subsection{Proof of the local concentration for agents}
Recall that the global concentration and corresponding global confidence bound 
have been shown in Lemma~\ref{lemma:global-concentration}.
Next, we establish the concentration properties of the local data on the agents' side.

\begin{proof}[Proof of Lemma~\ref{lemma:local-concentration}]
For each agent $m\in[M]$ and any rounds $1\leq t_1\leq t_2 \leq T$, consider
\begin{align*}
    \bSigma_{m,t_1,t_2} =  \alpha \lambda \Ib + \sum_{i=t_1+1, m_i=m}^{t_2} 
    \xb_i \xb_i^\top \quad \textnormal{and} \quad \ub_{m,t_1,t_2}
    = \sum_{i=t_1+1, m_i=m}^{t_2} \xb_i \eta_i.
\end{align*}
By Theorem 2 in \citet{abbasi2011improved}, with probability at least 
$1-\delta/(T^2 M)$, we have
\begin{align*}
    \|\bSigma_{m,t_1,t_2}^{-1} \ub_{m,t_1,t_2}\|_{\bSigma_{m,t_1,t_2}} \leq 
    R\sqrt{d\log \Big(\big(1+TL^2/(\alpha\lambda)\big)/\delta\Big)}+\sqrt{\lambda}S.
\end{align*}
Then taking an union bound over all agent $m\in[M]$ and rounds $1\leq t_1\leq 
t_2 \leq T$ and applying to $t_1=N_m(t)$ and $t_2=t$ for each $t \in [T]$, we 
obtain the desired concentration.
This finishes the proof.
\end{proof}

For clarity, we break Lemma~\ref{lemma:local_confidence_bound} into two lemmas, Lemma~\ref{lemma:concentration} for local confidence bound and Lemma~\ref{lemma:one-step-regret} for per-round regret, and prove them separately.

\begin{lemma}[Local confidence bound]\label{lemma:concentration}
Under the setting of Theorem~\ref{thm:regret}, with probability at least $1 - \delta$, for each $t \in [T]$, the estimate $\hat{\btheta}_{m,t+1}$ satisfies that $\|\btheta^* - \hat{\btheta}_{m,t+1}\|_{\bSigma_{m,t+1}} \leq \beta$.
\end{lemma}
\begin{proof}[Proof of Lemma~\ref{lemma:concentration}]
Since the estimated vector $\hat{\btheta}_{m,t+1}$ and covariance matrix 
$\bSigma_{m,t+1}$ will keep the same value as in the previous round if the agent 
$m$ do not communicate with the server, we only need to consider for those round 
$t$ where agent $m$ communicates with the server.
By the determinant-based criterion (Line~\ref{alg1:criterion}) in Algorithm \ref{alg: main}, if the agent $m$ 
communicates with the server in round $t$, then at the end of this round, the 
covariance matrix $\bSigma_{m,t+1}$ and vector $\bbb_{m,t+1}$ are given by
\begin{align}\label{eq:01}
    \bSigma_{m,t+1} = \lambda\Ib + \sum_{m'=1}^M  \bSigma_{m', N_{m'}(t)}^\up
    = \lambda \Ib + \sum_{m'=1}^M  \bSigma_{m',t}^\up, \qquad 
    \bbb_{m,t+1}= \sum_{m'=1}^M  \bbb_{m',t}^\up.
\end{align}
Therefore, the estimated vector $\hat{\btheta}_{m,t+1}$ is
\begin{align*}
    \hat{\btheta}_{m,t+1} &= \bigg(\lambda \Ib + \sum_{m'=1}^M \bSigma_{m',t}^\up
    \bigg)^{-1} \sum_{m'=1}^M \bbb_{m',t}^\up \notag\\
    &= \bigg(\lambda \Ib + \sum_{m'=1}^M \bSigma_{m',t}^\up\bigg)^{-1} 
    \sum_{m'=1}^M (\bSigma_{m',t}^\up \btheta^* + \ub_{m',t}^\up)\\
    &= \btheta^* - \lambda \bigg(\lambda \Ib + \sum_{m'=1}^M \bSigma_{m',t}^\up
    \bigg)^{-1} \btheta^* + \bigg(\lambda \Ib + \sum_{m'=1}^M 
    \bSigma_{m',k}^\up\bigg)^{-1} \sum_{m'=1}^M \ub_{m',t}^\up\\
    &= \btheta^* - \lambda (\bSigma_{m,t+1})^{-1} \btheta^* 
    + \sum_{m'=1}^M (\bSigma_{m,t+1})^{-1} \ub_{m',t}^\up.
\end{align*}
Thus, the difference between $\hat{\btheta}_{m,t+1}$ and the underlying truth 
$\btheta^*$ can be decomposed as
\begin{align}\label{eq:theta_diff}
    \big\|\btheta^* - \hat{\btheta}_{m,t+1}\big\|_{\bSigma_{m,t+1}} &\leq 
    \big\|\lambda (\bSigma_{m,t+1})^{-1} \btheta^*\big\|_{\bSigma_{m,t+1}} 
    + \bigg\|\sum_{m'=1}^M (\bSigma_{m,t+1})^{-1} \ub_{m',t}^\up\bigg\|_{\bSigma_{m,t+1}}\notag\\
    &\leq \sqrt{\lambda} \|\btheta^*\|_2 + \bigg\|\sum_{m'=1}^M (\bSigma_{m,t+1})^{-1}
    \ub_{m',t}^\up\bigg\|_{\bSigma_{m,t+1}},
\end{align}
where the first inequality holds due to that fact that 
$\|\ab+\bbb\|_{\bSigma}\leq \|\ab\|_{\bSigma}+\|\bbb\|_{\bSigma}$ and the second
inequality follows from $\bSigma_{m,t+1}\ge \lambda \Ib$. 
By the assumption that $\|\btheta^*\|_2\leq S$, the first term can be controlled 
by $\sqrt{\lambda} S$.
For the second term in \eqref{eq:theta_diff}, consider the following decomposition:
\begin{align}\label{eq:02}
    \sum_{m'=1}^M (\bSigma_{m,t+1})^{-1} \ub_{m',t}^\up &= \sum_{m'=1}^M 
    (\bSigma_{m,t+1})^{-1} \big(\ub_{m',t}^\up + \ub_{m',t}^\loc\big) 
    - \sum_{m'=1}^M (\bSigma_{m,t+1})^{-1} \ub_{m',t}^\loc \notag\\
    &= \underbrace{(\bSigma_{m,t+1})^{-1} \ub_t^\all}_{\cA} - \sum_{m'=1}^M 
    \underbrace{(\bSigma_{m,t+1})^{-1} \ub_{m',t}^\loc}_{B_{m'}}.
\end{align}
For the term $\cA$, it follows from \eqref{new2} in Lemma~\ref{lemma:covariance_comparison} that
\begin{align}\label{eq:04}
    \big\|(\bSigma_{m,t+1})^{-1} \ub_t^\all\big\|_{\bSigma_{m,t+1}} &= 
    \big\|(\bSigma_{m,t+1})^{-1/2} \ub_t^\all\big\|_2 \notag\\
    &\leq \sqrt{1+M\alpha} \cdot \big\|(\bSigma_t^\all)^{-1/2} \ub_t^\all\big\|_{2}\notag\\
    &\leq \sqrt{1+M\alpha} \cdot 
    \Big(R\sqrt{d\log \big((1+TL^2/\lambda)/\delta\big)}+\sqrt{\lambda}S\Big),
\end{align}
where the second inequality holds due to Lemma~\ref{lemma:global-concentration}.
Next, for each term $B_{m'}$ in \eqref{eq:02}, by \eqref{new1} in
Lemma~\ref{lemma:covariance_comparison}, we have
\begin{align*}
    \lambda \Ib + \sum_{j=1}^M \bSigma_{j,t}^\up \succeq \frac{1}{\alpha}\bSigma_{m',t}^\loc,
\end{align*}
which further implies that
\begin{align}
    \lambda \Ib + \sum_{j=1}^M \bSigma_{j,t}^\up \succeq \frac{1}{2\alpha} (
    \alpha \lambda \Ib + \bSigma_{m',t}^\loc).\label{eq:05}
\end{align}
Thus, the norm of each term $B_{m'}$ can be bounded as
\begin{align}\label{eq:06}
    \big\|(\bSigma_{m,t+1})^{-1} \ub_{m',t}^\loc\big\|_{\bSigma_{m,t+1}} &= 
    \big\|(\bSigma_{m,t+1})^{-1/2} \ub_{m',t}^\loc\big\|_2 \notag\\
    &\leq \sqrt{2\alpha} \cdot \Big\|\big(\alpha\lambda \Ib + 
    \bSigma_{m',t}^\loc\big)^{-1/2} \ub_{m',t}^\loc\Big\|_2 \notag\\
    &\leq \sqrt{2\alpha} \cdot \bigg(R\sqrt{d\log \frac{\alpha\lambda + TL^2}{\alpha\lambda\delta}}+\sqrt{\lambda}S\bigg),
\end{align}
where the first inequality holds due to~\eqref{eq:05} and the second inequality 
follows from Lemma~\ref{lemma:local-concentration}.

Finally, combining \eqref{eq:theta_diff}, \eqref{eq:02}, \eqref{eq:04}~and~\eqref{eq:06}, 
we obtain
\begin{align*}
\big\|\btheta^* - \hat{\btheta}_{m,t+1}\big\|_{\bSigma_{m,t+1}} 
&\leq \sqrt{\lambda} S + \big(\sqrt{1+M\alpha}+M\sqrt{2\alpha}\big) 
\bigg(R\sqrt{d\log \frac{\min(\alpha,1) \cdot \lambda +TL^2}{\min(\alpha,1)\cdot\lambda\delta}}
+\sqrt{\lambda}S \bigg).
\end{align*}
Thus we finish the proof of Lemma \ref{lemma:concentration}.
\end{proof}

\begin{lemma}[Per-round regret]\label{lemma:one-step-regret}
Under the setting of Theorem~\ref{thm:regret}, with probability at least $1 - \delta$, 
for each $t \in [T]$, the regret in round $t$ satisfies
\begin{align*}
    \Delta_t= \max_{\xb \in \cD_t}\la \btheta^*, \xb\ra - \la \btheta^*, \xb_t\ra\leq 2\beta\sqrt{\xb_t^{\top}\bSigma_{m_t,t}^{-1}\xb_t}.
\end{align*}
\end{lemma}
\begin{proof}[Proof of Lemma \ref{lemma:one-step-regret}]
First, by Lemma \ref{lemma:concentration}, with probability at least $1-\delta$, for each round $t\in [T]$ and each action $\xb \in \cD_t$, we have 
\begin{align}
   \hat{\btheta}_{m,t}^{\top}\xb+\beta\sqrt{\xb^{\top}\bSigma_{m_t,t}^{-1} \xb}- (\btheta^*)^{\top}\xb&=(\hat{\btheta}_{m,t}-\btheta^*)^{\top}\xb+\beta\sqrt{\xb^{\top}\bSigma_{m_t,t}^{-1} \xb}\notag\\
    &\ge -\|\hat{\btheta}_{m,t}-\btheta^*\|_{\bSigma_{m_t,t}} \cdot \|\xb\|_{\bSigma_{m_t,t}^{-1}}+\beta\sqrt{\xb^{\top}\bSigma_{m_t,t}^{-1} \xb}\notag\\
    &\ge -\beta \|\xb\|_{\bSigma_{m_t,t}^{-1}}+\beta\sqrt{\xb^{\top}\bSigma_{m_t,t}^{-1}\xb}\notag\\
    &=0,\label{eq:2-2}
\end{align}
where the first inequality holds due to the Cauchy-Schwarz inequality and the 
last inequality follows from Lemma \ref{lemma:concentration}. 
\eqref{eq:2-2} shows that the estimator for agent $m_t$ is always optimistic. 
For simplicity, we denote the optimal action at round $t$ as 
$\xb^*=\arg \max_{\xb \in \cD_t}(\btheta^*)^{\top}\xb$, and \eqref{eq:2-2} further implies 
\begin{align}
    \Delta_t&=(\btheta^*)^{\top}\xb^*-(\btheta^*)^{\top}\xb_t\notag\\
    &\leq \hat{\btheta}_{m,t}^{\top}\xb^*+\beta\sqrt{(\xb^*)^{\top}\bSigma_{m_t,t}^{-1} \xb^*}-(\btheta^*)^{\top}\xb_t\notag\\
    & \leq \hat{\btheta}_{m,t}^{\top}\xb_t+\beta\sqrt{\xb_t^{\top}\bSigma_{m_t,t}^{-1} \xb_t}-(\btheta^*)^{\top}\xb_t\notag\\
    &= (\hat{\btheta}_{m,t}-\btheta^*)^{\top}\xb_t+\beta\sqrt{\xb_t^{\top}\bSigma_{m_t,t}^{-1} \xb_t}\notag\\
    &\leq \|\hat{\btheta}_{m,t}-\btheta^*\|_{\bSigma_{m_t,t}} \cdot \|\xb_{k}\|_{\bSigma_{m_t,t}^{-1}}+\beta\sqrt{\xb_t^{\top}\bSigma_{m_t,t}^{-1} \xb_t}\notag\\
    &\leq 2 \beta \sqrt{\xb_t^{\top}\bSigma_{m_t,t}^{-1} \xb_t},\notag
\end{align}
where the first inequality holds due to \eqref{eq:2-2}, the second inequality 
follows from the definition of action $\xb_t$ in Algorithm \ref{algorithm:1}, 
the third inequality applies the Cauchy-Schwarz inequality, and the last 
inequality is by Lemma~\ref{lemma:concentration}. 
Thus, we finish the proof of Lemma \ref{lemma:one-step-regret}.
\end{proof}

Combining Lemmas~\ref{lemma:concentration} and~\ref{lemma:one-step-regret} yields Lemma~\ref{lemma:local_confidence_bound}

\section{Proof for Lower Bound}\label{apdx:lower_bound}
\begin{lemma}[Theorem 3 in \citealt{abbasi2011improved}]\label{lemma:1-up}
There exists a constant $C>0$, such that for any normalized linear bandit instance 
with $R=L=S=1$, the expectation of the regret for OFUL algorithm is upper bounded by
$\EE[\text{Regret}(T)]\leq Cd\sqrt{T}\log T$.
\end{lemma}

\begin{lemma}[Theorem 24.1 in \citealt{lattimore2020bandit}]\label{lemma:1-low}
There exists a set of hard-to-learn normalized linear bandit instances with 
$R=L=S=1$, such that for any algorithm \textbf{Alg} and $T\ge d$, for a uniformly 
random instance in the set, the regret is lower bounded by
$\EE[\text{Regret}(T)]\ge cd\sqrt{T}$ for some constant $c>0$.
\end{lemma}

Theorem~\ref{thm:low} is an extension of the lower bound result in 
\citet[Theorem 2]{wang2019distributed} from multi-arm bandits to linear bandits.

\begin{proof}[Proof of Theorem~\ref{thm:low}]
For any algorithm \textbf{Alg} for federated bandits, we construct the auxiliary 
\textbf{Alg1} as follows: For each agent $m\in[M]$, it performs \textbf{Alg} 
until there is a communication between the agent $m$ and the server (upload or download data). 
After the communication, the agent $m$ remove all previous information and 
perform the \texttt{OFUL} Algorithm in \citet{abbasi2011improved}. 
In this case, for each agent $m\in[M]$, \textbf{Alg1} do not utilize any 
information from other agents and it will reduce to a single agent bandit algorithm.

Now, we uniformly randomly select a hard-to-learn instance from the set given 
in Lemma \ref{lemma:1-low}, and let each agent $m\in[M]$ be active for $T/M$ 
different rounds (where we assume $T/M$ is an integer for simplicity). 
Since \textbf{Alg1} reduces to a single agent bandit algorithm, 
Lemma \ref{lemma:1-low} implies that the expected regret for agent $m$ with \textbf{Alg1} 
is lower bounded by
\begin{align}
    \EE[\text{Regret}_{m,\textbf{Alg1}}]\ge cd\sqrt{T/M}.\label{eq:0001}
\end{align}
Taking the sum of \eqref{eq:0001} over all agents $m\in[M]$, we obtain
\begin{align}
    \EE[\text{Regret}_{\textbf{Alg1}}]
    &=\sum_{m=1}^M \EE[\text{Regret}_{m,\textbf{Alg1}}]\ge cd\sqrt{MT}.\label{eq:0003}
\end{align}

For each agent $m\in[M]$, let $\delta_m$ denote the probability that agent $m$ 
will communicate with the server. 
Notice that before the communication, \textbf{Alg1} has the same performance as 
\textbf{Alg}, while for the rounds after the communication, \textbf{Alg1} 
executes the \texttt{OFUL} algorithm and Lemma \ref{lemma:1-up} suggests an 
$O(d\sqrt{T/M}\log (T/M))$ upper bounded for the expected regret. 
Therefore, the expected regret for agent $m$ with \textbf{Alg1} is upper bounded by
\begin{align}
    \EE[\text{Regret}_{m,\textbf{Alg1}}]
    \leq \EE[\text{Regret}_{m,\textbf{Alg}}]
    + \delta_m C d\sqrt{T/M}\log (T/M).\label{eq:0002}
\end{align}
Taking the sum of \eqref{eq:0002} over all agents $m\in[M]$, we obtain
\begin{align}
    \EE[\text{Regret}_{\textbf{Alg}}(T)]&=\sum_{m=1}^M\EE[\text{Regret}_{m,\textbf{Alg1}}]\notag\\
    &\leq  \sum_{m=1}^M \EE[\text{Regret}_{m,\textbf{Alg}}]
    + \bigg(\sum_{m=1}^M \delta_m\bigg)  C d\sqrt{T/M}\log (T/M)\notag\\
    &= \EE[\text{Regret}_{\textbf{Alg}}] + \delta Cd\sqrt{T/M}\log(T/M),\label{eq:0004}
\end{align}
where $\delta=\sum_{m=1}^M \delta_m$ is the expected communication complexity. 
Combining \eqref{eq:0003} and \eqref{eq:0004}, for any algorithm \textbf{Alg} 
with communication complexity $\delta\leq c/(2C) \cdot M/\log(T/M)=O(M/\log(T/M))$, we have
\begin{align*}
    \EE[\text{Regret}_{\textbf{Alg}}]&\ge cd\sqrt{MT}- \delta Cd\sqrt{T/M}\log(T/M)
    \ge cd\sqrt{MT}/2 = \Omega (d\sqrt{MT}).
\end{align*}
This finishes the proof of Theorem~\ref{thm:low}.
\end{proof}

\section{Auxiliary Lemmas}
To make the analysis self-contained in this paper, here we include the auxiliary 
lemmas that have been previously used. 

\begin{lemma}[Lemma 2.2 in \citealt{tie2011rearrangement}]\label{lem:det1}
For any positive semi-definite matrices $\Ab$, $\Bb$ and $\Cb$, it holds that
$\det(\Ab+\Bb+\Cb)+\det(\Ab)\ge \det(\Ab+\Bb)+\det(\Ab+\Cb)$.
\end{lemma}

\begin{lemma}[Lemma 2.3 in \citealt{tie2011rearrangement}]\label{lem:det2}
For any positive semi-definite matrices $\Ab$, $\Bb$ and $\Cb$, it holds that
$\det(\Ab+\Bb+\Cb)\det(\Ab)\leq \det(\Ab+\Bb)\det(\Ab+\Cb)$.
\end{lemma}

\begin{theorem}[Theorem 2 in \citealt{abbasi2011improved}]\label{thm: confidence ellipsoid}
    Suppose $\{\cF_t\}_{t=0}^\infty$ is a filtration.
    Let $\{\eta_t\}_{t=1}^\RR$ be a stochastic process in $\RR$ such that $\eta_t$ is $\cF_t$-measurable and $R$-sub-Gaussian conditioning on $\cF_{t-1}$, i.e, for any $c>0$,
    \begin{align*}
        \EE\left[ \exp\left( c \eta_t \right)\middle| \cF_{t-1} \right] 
        \leq \exp\bigg( \frac{c^2 R^2}{2} \bigg).
    \end{align*}
    Let $\{\xb_t\}_{t=1}^\infty$ be a stochastic process in $\RR^d$ such that $\xb_t$ is $\cF_{t-1}$-measurable and $\|\xb_t\|_2 \leq L$. 
    Let $y_t = \langle \xb_t, \btheta^* \rangle + \eta_t$ for some $\btheta^* \in \RR^d$ s.t. $\|\btheta^*\|_2\leq S$.  
    For any $t \geq 1$, define
    \begin{align*}
        \bSigma_t = \lambda \Ib + \sum_{i=1}^t \xb_t \xb_t^\top, \ \  \textnormal{and} \quad \hat\btheta_t = \bSigma_t^{-1} \sum_{i=1}^t \xb_i y_i,
    \end{align*}for some $\lambda>0$. 
    Then for any $\delta>0$, with probability at least $1-\delta$, for all $t$, we have 
    \begin{align*}
        \|\hat\btheta_t-\btheta^*\|_{\bSigma_t} 
        \leq R \sqrt{d \log\bigg( \frac{1+tL^2/\lambda}{\delta}\bigg)}+\sqrt{\lambda}S.
    \end{align*}
\end{theorem}
\begin{lemma}[Lemma 12 in \citet{abbasi2011improved}]\label{lemma:det}
Suppose $\Ab, \Bb\in \RR^{d \times d}$ are two positive definite matrices satisfying that $\Ab \succeq \Bb$, then for any $\xb \in \RR^d$, $\|\xb\|_{\Ab} \leq \|\xb\|_{\Bb}\cdot \sqrt{\det(\Ab)/\det(\Bb)}$.
\end{lemma}

\bibliographystyle{ims}
\bibliography{reference}
\end{document}